\documentclass[twocolumn, switch]{article} 

\usepackage{preprint}

\usepackage{amsmath, amsthm, amssymb, amsfonts}

\usepackage[numbers,square]{natbib}
\usepackage[utf8]{inputenc}	
\usepackage[T1]{fontenc}	
\usepackage{xcolor}		
\usepackage[colorlinks = true,
            linkcolor = purple,
            urlcolor  = blue,
            citecolor = cyan,
            anchorcolor = black]{hyperref}	
\usepackage{booktabs} 		
\usepackage{nicefrac}		
\usepackage{microtype}		
\usepackage{lineno}		
\usepackage{float}			

\usepackage{lipsum}		

\usepackage{colortbl,xcolor}
\usepackage{pifont}   
\usepackage{amssymb}  
\newcommand{\cmark}{\textcolor{green!60!black}{\ding{51}}}
\newcommand{\xmark}{\textcolor{red!70!black}{\ding{55}}}

\definecolor{bestcolor}{RGB}{0,120,0}      
\definecolor{secondcolor}{RGB}{180,90,0}   

\newcommand{\best}[1]{\textbf{\textcolor{bestcolor}{#1}}}
\newcommand{\second}[1]{\textcolor{secondcolor}{\underline{#1}}}

\newtheorem{theorem}{Theorem}
\newtheorem{corollary}{Corollary}

\usepackage{newfloat}
\DeclareFloatingEnvironment[name={Supplementary Figure}]{suppfigure}
\usepackage{sidecap}
\sidecaptionvpos{figure}{c}

\usepackage{titlesec}
\titlespacing\section{0pt}{12pt plus 3pt minus 3pt}{1pt plus 1pt minus 1pt}
\titlespacing\subsection{0pt}{10pt plus 3pt minus 3pt}{1pt plus 1pt minus 1pt}
\titlespacing\subsubsection{0pt}{8pt plus 3pt minus 3pt}{1pt plus 1pt minus 1pt}

\usepackage{tikz,xcolor,hyperref}

\definecolor{lime}{HTML}{A6CE39}
\DeclareRobustCommand{\orcidicon}{
	\begin{tikzpicture}
	\draw[lime, fill=lime] (0,0) 
	circle [radius=0.16] 
	node[white] {{\fontfamily{qag}\selectfont \tiny ID}};
	\draw[white, fill=white] (-0.0625,0.095) 
	circle [radius=0.007];
	\end{tikzpicture}
	\hspace{-2mm}
}
\foreach \x in {A, ..., Z}{\expandafter\xdef\csname orcid\x\endcsname{\noexpand\href{https://orcid.org/\csname orcidauthor\x\endcsname}
			{\noexpand\orcidicon}}
}

\title{HopFormer: Sparse Graph Transformers \\with Explicit Receptive Field Control}
\newcommand\titleheader{HopFormer: Sparse Graph Transformers with Explicit Receptive Field Control}

\usepackage{xwatermark}

\usepackage{authblk}

\author{Sanggeon Yun}
\author{Raheeb Hassan}
\author{Ryozo Masukawa}
\author{Sungheon Jeong}
\author{Mohsen Imani}

\affil{Department of Computer Science, University of California, Irvine\\\texttt{\{sanggeoy, raheebh, rmasukaw, sungheoj, m.imani\}@uci.edu}}

\begin{document}

\twocolumn[ 
  \begin{@twocolumnfalse} 
  
\maketitle

\begin{abstract}
    Graph Transformers typically rely on explicit positional or structural encodings and dense global attention to incorporate graph topology. In this work, we show that neither is essential. We introduce HopFormer, a graph Transformer that injects structure exclusively through head-specific $n$-hop masked sparse attention, without the use of positional encodings or architectural modifications. This design provides explicit and interpretable control over receptive fields while enabling genuinely sparse attention whose computational cost scales linearly with mask sparsity. Through extensive experiments on both node-level and graph-level benchmarks, we demonstrate that our approach achieves competitive or superior performance across diverse graph structures. Our results further reveal that dense global attention is often unnecessary: on graphs with strong small-world properties, localized attention yields more stable and consistently high performance, while on graphs with weaker small-world effects, global attention offers diminishing returns. Together, these findings challenge prevailing assumptions in graph Transformer design and highlight sparsity-controlled attention as a principled and efficient alternative.
\end{abstract}
\vspace{0.35cm}

  \end{@twocolumnfalse} 
] 



\section{Introduction}
\begin{table*}[t]
\centering
\caption{
Comparison of graph learning architectures by core design components.
\cmark\ indicates preservation of the vanilla Transformer architecture
(i.e., standard multi-head self-attention and feed-forward blocks without architectural modifications),
while \xmark\ denotes architectural modifications.
Here, $|V|$ and $|E|$ denote the number of nodes and edges in the input graph, respectively.
``Flexible'' complexity indicates that the computational cost adapts to graph sparsity and head-specific receptive fields,
scaling at least linearly with $|V|+|E|$ under sparse attention, rather than incurring dense quadratic cost.
}
\label{tab:design_comparison}
\resizebox{\linewidth}{!}{
\begin{tabular}{lccccc}
\toprule
\textbf{Method}
& \textbf{Structural / Positional Encoding}
& \textbf{Vanilla Transformer}
& \textbf{Attention Pattern}
& \textbf{Complexity per Layer} \\
\midrule
\rowcolor{gray!15}
MPNNs
& --
& --
& Sparse (local)
& $\mathcal{O}(|E|)$ \\

\midrule
Graphormer~\cite{ying2021transformers}
& SPD, centrality
& \xmark
& Dense Global Attention
& $\mathcal{O}(|V|^2)$ \\

SAN~\cite{kreuzer2021rethinking}
& Laplacian eigenvectors
& \xmark
& Dense Global Attention
& $\mathcal{O}(|V|^2)$ \\

GraphGPS~\cite{rampasek2022graphgps}
& Structural encodings
& \xmark
& Dense Global Attention
& $\mathcal{O}(|V|^2)$ \\

NodeFormer~\cite{wu2022nodeformer}
& Random features
& \xmark
& Approximated Global Attention
& $\mathcal{O}(|V| + |E|)$  \\

TokenGT~\cite{kim2022pure}
& Structural / edge encodings
& \cmark
& Dense Global Attention
& $\mathcal{O}\big((|V|+|E|)^2\big)$ \\

SpecFormer~\cite{bo2023specformer}
& Spectral features
& \xmark
& Dense Global Attention
& $\mathcal{O}(|V|^2)$ \\

DIFFormer~\cite{wu2023difformer}
& Diffusion kernels
& \xmark
& Dense Global Attention
& $\mathcal{O}(|V|^2)$ \\

Exphormer~\cite{shirzad2023exphormer}
& Subgraph encodings
& \xmark
& Sparse Attention
& $\mathcal{O}(|E|)$ \\

GRIT~\cite{ma2023grit}
& Learned structural tokens
& \xmark
& Dense Global Attention
& $\mathcal{O}(|V|^2)$ \\

ESA~\cite{buterez2025end}
& None
& \xmark
& Masked+Dense Global Attention
& $\mathcal{O}(|E|^2)$ \\

\midrule
\rowcolor{blue!8}
\textbf{Ours (HopFormer)}
& \textbf{None}
& \cmark
& \textbf{Sparse Attention Guided by $n$-Hop Masks}
& Flexible with $\Omega\big(|V| + |E|\big)$\\

\bottomrule
\end{tabular}
}
\end{table*}

Transformers~\cite{vaswani2017attention} have become a dominant architecture in modern machine learning due to their expressive self-attention mechanism and strong scalability with respect to data and model size. By relying on weak inductive biases, Transformers learn task-relevant structure directly from data, a property that has driven their success in domains characterized by long-range dependencies and heterogeneous interactions, including natural language processing~\cite{brown2020language} and vision~\cite{arnab2021vivit}.

Extending Transformers to graph-structured data has therefore attracted growing interest as an alternative to Message-Passing Graph Neural Networks (MPNNs)~\cite{gilmer2017neural}. While MPNNs leverage sparse local aggregation and achieve strong empirical performance, their strictly local propagation limits the ability to capture long-range dependencies and global structural patterns, leading to phenomena such as over-smoothing~\cite{oonograph} and information bottlenecks in deep architectures~\cite{topping2022understanding}. These limitations have motivated the adoption of Transformer-style attention to enable broader interaction ranges on graphs.

Applying Transformers to graphs, however, is fundamentally more complex than in sequential domains. Graphs are non-sequential, permutation-invariant structures~\cite{keriven2019universal} with heterogeneous information sources, including node attributes, edge attributes, and relational topology~\cite{hu2020heterogeneous}. As a result, most existing graph Transformer models modify the standard Transformer design to explicitly inject structural information. Common strategies include adding positional or structural encodings—such as shortest-path distances, centrality measures, Laplacian eigenvectors, diffusion kernels, or random-walk statistics—into attention logits or node embeddings~\cite{zhang2020graph,ying2021transformers,kreuzer2021rethinking,rampasek2022graphgps,ma2023grit}. Other approaches introduce architectural modifications, such as hybrid message-passing and attention blocks or auxiliary structural tokens~\cite{rampasek2022graphgps,ma2023grit,kim2022pure,wu2023difformer,buterez2025end}. A systematic comparison of these design choices is summarized in \autoref{tab:design_comparison}.

Despite their strong empirical performance, these developments raise fundamental questions about the necessity of prevailing graph Transformer design choices.
First, \textbf{is it truly necessary to inject graph structure through explicit positional or structural encodings, or through architectural modifications of the Transformer itself?} Such encodings are highly sensitive to their formulation, dimensionality, scaling, and point of integration~\cite{haviv2022transformer,zheng2024dape,wang2024length}, and they can impair generalization when test-time graph sizes or structural distributions differ from those observed during training~\cite{wang2024length}. In contrast, recent results in sequence modeling demonstrate that Transformers can recover both absolute and relative positional information using attention masking alone, even without explicit positional embeddings~\cite{haviv2022transformer,kazemnejad2023impact,wang2024length}. These findings suggest that structural inductive bias may be imposed implicitly through constrained attention patterns, without relying on handcrafted encodings or auxiliary architectural components.

Second, \textbf{is dense global self-attention essential for effective graph representation learning?} Although many graph Transformers retain global attention over all token pairs, real-world graphs frequently exhibit strong small-world properties~\cite{watts1998collective,zitin2014spatially}, where informative interactions are concentrated within a limited number of hops. Moreover, global attention incurs quadratic or worse computational cost. When both nodes $V$ and edges $E$ are treated as tokens—a common strategy for modeling heterogeneous graph information~\cite{ying2021transformers,kim2022pure}—the attention matrix scales as $(|V|+|E|)\times(|V|+|E|)$, leading to prohibitive $\mathcal{O}(|V|^4)$ complexity in dense graphs. This raises the question of whether dense global attention is necessary at all, or whether carefully structured, locality-aware attention mechanisms can suffice to capture both local and long-range dependencies.

To answer these questions, we propose HopFormer, a new graph Transformer paradigm that adopts the simplest design principles proven effective in sequential Transformer models. Our approach injects graph structure exclusively through \emph{head-specific $n$-hop masked sparse attention}, in which each attention head is restricted to a prescribed $n$-hop receptive field via attention masks, without relying on positional or structural encodings and without modifying the standard Transformer architecture. Importantly, these masks are applied \emph{prior} to attention computation, so self-attention is performed only on the masked support, yielding genuinely sparse attention rather than masked dense attention. By explicitly controlling the hop radius per head, our design exposes a direct and interpretable trade-off between receptive field size and computational cost. To jointly model node and edge information without auxiliary encodings, we employ a node–edge tokenization scheme based on an augmented incidence graph, treating edges as auxiliary tokens connected to their endpoints. Node and edge attributes are embedded using lightweight, modality-specific input projections into a shared latent space, following established practice in multimodal Transformers~\cite{deitke2025molmo,wu2025janus}. This formulation enables unified attention over heterogeneous graph components while preserving architectural simplicity.

Our theoretical analysis and extensive experiments demonstrate that explicit positional or structural encodings are not a prerequisite for effective graph Transformers. Instead, we show that graph structure can be fully and efficiently injected through head-specific $n$-hop masked sparse attention alone, without modifying the standard Transformer architecture. Moreover, our results reveal that dense global attention is often unnecessary: on graphs with strong small-world properties, carefully controlled local receptive fields not only suffice but yield more stable and consistently strong performance across datasets. Even on graphs with weaker small-world effects, global attention exhibits diminishing returns, with multiple Transformer variants converging to similar performance. Together, these findings establish that explicit, topology-aligned control of attention sparsity offers a principled and generalizable alternative to increasingly complex graph Transformer designs, achieving competitive accuracy with improved interpretability and computational efficiency.

\begin{figure*}[th]
    \centering
    \includegraphics[width=1.0\linewidth]{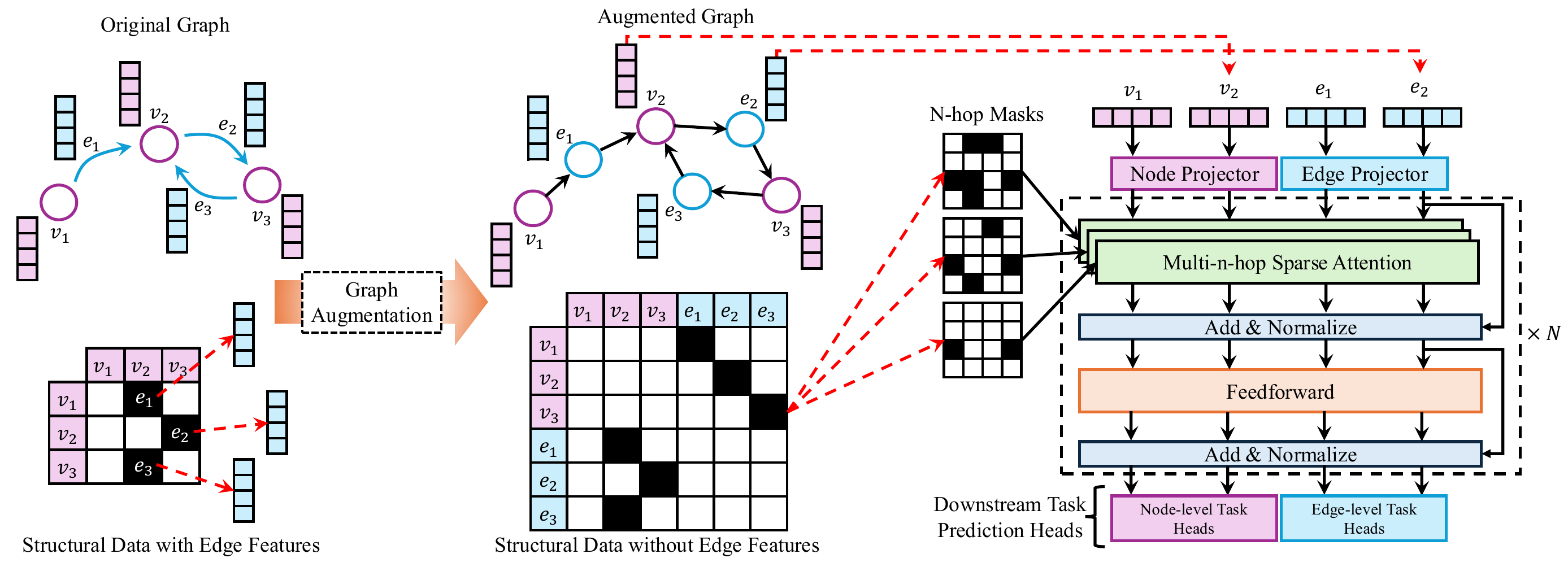}
    \caption{An illustration of the proposed edge-to-node augmentation, node–edge tokenization, and head-specific $n$-hop sparse attention mechanism.}
    \label{fig:overview}
\end{figure*}

\section{Related Work and Background}

\paragraph{Attention Masking and Structural Bias.}
Attention masking provides a principled mechanism for constraining information flow in Transformers without altering the attention operator itself. In sequence models, causal masking has been shown to implicitly induce positional structure, allowing Transformers to recover absolute and relative positional relationships even without explicit positional embeddings~\cite{haviv2022transformer,kazemnejad2023impact,wang2024length}. Motivated by these findings, we adopt masking as the sole mechanism for injecting structural inductive bias in graph Transformers. Specifically, we employ head-specific $n$-hop attention masks that bound each attention head’s receptive field according to graph topology, providing explicit and interpretable control over information propagation while preserving the standard self-attention formulation.

\paragraph{Sparse Attention and Sparse Computation.}
Sparse matrix operations enable efficient computation when interactions are inherently sparse, with time complexity proportional to the number of nonzero entries, i.e., $\mathcal{O}(\mathrm{nnz})$, rather than quadratic in problem size~\cite{yang2018design}.
In graph learning, MPNNs naturally leverage sparsity by formulating neighborhood aggregation as sparse matrix operations~\cite{gilmer2017neural,jiang2019semi}. However, many graph Transformers retain dense global self-attention over all $K$ tokens, incurring $\mathcal{O}(K^2)$ complexity regardless of graph sparsity. Even when attention masks are employed, dense attention scores are often computed prior to masking, resulting in masked dense attention rather than true sparsity. In contrast, our approach applies masks \emph{before} attention computation, enabling attention to be implemented directly via sparse operations with complexity scaling as $\mathcal{O}(\mathrm{nnz}(\mathbf{M}))$, where $\mathbf{M}$ denotes the attention mask.

\paragraph{Small-World Property.}
Given a graph $G=(V,E)$, where $d(u,v)$ denotes the shortest-path distance between nodes $u,v\in V$, many real-world graphs exhibit the \emph{small-world} property, characterized by a small average shortest-path length
\begin{equation}
\ell(G) = \frac{1}{|V|(|V|-1)} \sum_{u \neq v} d(u,v) = \mathcal{O}(\log |V|),
\end{equation}
together with a high clustering coefficient
\begin{equation}
C(G) = \frac{1}{|V|} \sum_{v \in V} \frac{2\,|\mathcal{N}(v)|}{\deg(v)\,(\deg(v)-1)},
\end{equation}
where $\mathcal{N}(v)$ denotes the set of edges between neighbors of $v$ and $\deg(v)$ is the degree of $v$.
Such properties are widely observed in social, biological, and information networks~\cite{watts1998collective,zitin2014spatially}.
They imply that informative interactions are typically concentrated within a small number of hops, rendering dense global attention unnecessary at every layer and motivating attention mechanisms that explicitly restrict receptive fields while retaining the ability to capture long-range dependencies.

\section{Methodology}

In this section, we introduce HopFormer, a graph Transformer that injects structure solely through head-specific attention masks while leaving the standard Transformer architecture intact (\autoref{fig:overview}). Our design enables explicit control over attention receptive fields, avoiding redundant global interactions on small-world graphs, and preserves sparse attention computation for improved efficiency—without relying on positional or structural encodings.

\subsection{Edge-to-Node Augmentation via Incidence Graph}

Let $G=(V,E)$ be a graph with $|V|=N$ nodes and $|E|=M$ edges. Each node $v\in V$ is associated with a feature vector $\mathbf{x}_v\in\mathbb{R}^{d_v}$, and each edge $e\in E$ may optionally carry features $\mathbf{e}_e\in\mathbb{R}^{d_e}$.
To enable a unified treatment of nodes and edges as tokens in a Transformer, we first convert $G$ into an \emph{augmented incidence graph} $\widetilde{G}=(\widetilde{V},\widetilde{E})$ by converting each edge into an \emph{edge node}. Concretely, let $\widetilde{V}=V\cup E$ where we identify each $e=(u,v)\in E$ with a new vertex also denoted $e$. We then connect each edge-node to its endpoint nodes:
\begin{equation}
\widetilde{E}=\{(u,e),(e,u),(v,e),(e,v)\mid e=(u,v)\in E\}.
\end{equation}
Thus each original edge becomes two undirected incidence links (or four directed links), and the augmented graph has $|\widetilde{V}|=N+M$ and $|\widetilde{E}|=\Theta(M)$. Importantly, this transformation preserves sparsity up to a constant factor and does not alter the asymptotic edge complexity of the original graph.

\subsection{Input Projectors for Node--Edge Tokenization}

Based on the augmented incidence graph $\widetilde{G}$, we represent the augmented node set $\widetilde{V}$ as a collection of tokens consisting of node tokens, corresponding to $V\cap\widetilde{V}$, and edge tokens, corresponding to $E\cap\widetilde{V}$. To accommodate heterogeneous feature spaces without modifying the Transformer layers, we employ lightweight \emph{input projectors} to map node and edge attributes into a shared $d$-dimensional token space:
\begin{equation}
\mathbf{h}_v=\mathbf{W}_n\mathbf{x}_v,\quad v\in V\cap\widetilde{V},
\qquad
\mathbf{h}_e=\mathbf{W}_e\mathbf{e}_e,\quad e\in E\cap\widetilde{V},
\end{equation}
with $\mathbf{W}_n\in\mathbb{R}^{d\times d_v}$ and $\mathbf{W}_e\in\mathbb{R}^{d\times d_e}$. If edge features are unavailable, we set $\mathbf{e}_e=\mathbf{0}$ so that edge tokens contribute only through topology. The resulting input token sequence is
\begin{equation}
\mathbf{H}^{(0)}=\big[\{\mathbf{h}_v\}_{v\in V\cap\widetilde{V}};\{\mathbf{h}_e\}_{e\in E\cap\widetilde{V}}\big]\in\mathbb{R}^{(N+M)\times d}.
\end{equation}
This construction treats nodes and edges symmetrically as tokens and avoids auxiliary modules (e.g., cross-attention) or separate update rules.

\subsection{Head-Specific $n$-Hop Sparse Attention}

We inject graph structure into the Transformer exclusively through \emph{head-specific attention masks} derived from the augmented incidence graph $\widetilde{G}$. These masks depend only on the connectivity encoded by $\widetilde{E}$ and do not incorporate edge attributes or any additional structural features, leaving all other components of the Transformer architecture unchanged.

Let $\widetilde{\mathbf{A}}\in\{0,1\}^{(N+M)\times(N+M)}$ denote the sparse adjacency matrix of $\widetilde{G}$. For each attention head $h\in\{1,\dots,H\}$, we assign a hop budget $n_h\in\mathbb{N}$, which defines the maximum receptive field of that head in terms of graph distance.

Conceptually, we define the head-specific attention mask as the $n_h$-hop reachability matrix
\begin{equation}\label{eq:defmask}
\mathbf{M}^{(h)} = \mathbb{I}\!\left[\sum_{k=0}^{n_h} \widetilde{\mathbf{A}}^{\,k} > 0\right]
\;\in\; \{0,1\}^{(N+M)\times(N+M)},
\end{equation}
where $\widetilde{\mathbf{A}}^{0}=\mathbf{I}$ and the indicator function is applied elementwise. An entry $\mathbf{M}^{(h)}_{ij}=1$ indicates that token $j$ is reachable from token $i$ within at most $n_h$ incidence hops in $\widetilde{G}$.

Given token embeddings $\mathbf{H}\in\mathbb{R}^{(N+M)\times d}$, attention head $h$ computes scaled dot-product attention restricted to the support of $\mathbf{M}^{(h)}$:
\begin{equation}\label{eq:dotproduct}
\mathrm{Attn}^{(h)}(\mathbf{Q},\mathbf{K},\mathbf{V})
=
\mathrm{softmax}\!\left(
\frac{ (\mathbf{Q}\mathbf{K}^\top)_{\mathbf{M}^{(h)}} }{\sqrt{d_h}}
\right)\mathbf{V},
\end{equation}
where $d_h=d/H$ and $(\cdot)_{\mathbf{M}^{(h)}}$ denotes restriction to the nonzero entries specified by the mask. Equivalently, attention scores are computed \emph{only} for token pairs whose graph distance is at most $n_h$, while all other interactions are excluded from both computation and normalization.

\paragraph{Sparse attention implementation.}
Because attention is restricted \emph{prior} to score computation, the resulting operation constitutes genuine sparse attention rather than masked dense attention. The number of attention interactions per head is determined by the size of the $n_h$-hop neighborhoods in $\widetilde{G}$ and therefore depends on graph topology and hop budget, rather than on the total number of tokens $(N+M)$. This allows attention to be implemented directly over sparse supports, with computational cost proportional to $\mathrm{nnz}(\mathbf{M}^{(h)})\cdot d_h$. Consequently, each head explicitly trades receptive field size for computational cost through its hop budget $n_h$, yielding interpretable and topology-aligned control over attention.

\subsection{Overall Transformer Architecture}

Our model follows the standard Transformer encoder architecture~\cite{vaswani2017attention} without modifying its layer structure, residual connections, normalization, or feed-forward networks. Graph structure is injected \emph{exclusively} through the multi-head self-attention module by replacing dense global attention with head-specific masked attention, while all other components remain identical to the vanilla Transformer.

Let $\mathbf{Z}^{(\ell)}\in\mathbb{R}^{T\times d}$ denote the token representations at layer $\ell$, where $T=N+M$. Each layer is defined as
\begin{align}
\widetilde{\mathbf{Z}}^{(\ell)}
&=
\underbrace{\mathrm{MHSA}_{\{\mathbf{M}^{(h)}\}_{h=1}^H}\!\left(\mathbf{Z}^{(\ell)}\right)}_{\textbf{only difference from vanilla Transformer}}
+\mathbf{Z}^{(\ell)}, \label{eq:masked_block}\\
\mathbf{Z}^{(\ell+1)}
&=
\mathrm{FFN}\!\left(\widetilde{\mathbf{Z}}^{(\ell)}\right)
+\widetilde{\mathbf{Z}}^{(\ell)}. \label{eq:ffn_block}
\end{align}

Here, $\mathrm{MHSA}_{\{\mathbf{M}^{(h)}\}_{h=1}^H}$ denotes masked multi-head self-attention, in which each head $h$ computes attention using the head-specific $n_h$-hop mask $\mathbf{M}^{(h)}$ via the masked attention operator defined in \autoref{eq:dotproduct}. Apart from the inclusion of these masks, the attention computation and head aggregation are identical to those in the standard Transformer~\cite{vaswani2017attention}.

\subsection{Training with Downstream Tasks}

Let $\{\mathbf{h}^{(L)}_i\}=\mathbf{H}^{(L)} \in \mathbb{R}^{(N+M)\times d}$ denote the output token representations from the final Transformer layer.  
For node-level tasks, predictions are obtained by applying a shared classifier $f_{\mathrm{node}}$ to node tokens:
\begin{equation}
\hat{\mathbf{y}}_v = f_{\mathrm{node}}\!\left(\mathbf{h}^{(L)}_v\right), \quad v \in V\cap\widetilde{V},
\end{equation}
and the model is trained using standard node-wise losses, such as cross-entropy for classification.

For graph-level tasks, a permutation-invariant readout function $\rho(\cdot)$ aggregates token representations into a single graph representation:
\begin{equation}
\mathbf{h}_G = \rho\!\left(\{\mathbf{h}^{(L)}_v\}_{v\in V\cap\widetilde{V}} \cup \{\mathbf{h}^{(L)}_e\}_{e\in E\cap\widetilde{V}}\right),
\end{equation}
which is then mapped to the prediction $\hat{\mathbf{y}}_G = f_{\mathrm{graph}}(\mathbf{h}_G)$ using a task-specific head. Typical choices for $\rho$ include sum or mean pooling.

All parameters, including input projectors and Transformer layers, are optimized jointly via backpropagation by minimizing the downstream loss. Since graph structure is incorporated solely through head-specific attention masks, the training procedure is identical to that of a standard Transformer, while yielding topology-aware and sparsity-preserving representations.

\section{Theoretical Analysis}

We provide a theoretical analysis of the proposed HopFormer to justify two central claims:
(i) injecting graph structure solely through head-specific attention masks is sufficient to convey structural information to the Transformer, and
(ii) equipping a single Transformer layer with multiple heads of different receptive fields strictly increases its expressiveness.

\paragraph{Assumption 1 (Full-rank projections).}
Following standard assumptions in expressiveness analyses of attention-based models
(e.g.,~\cite{yun2019transformers,cordonnier2019relationship}),
we assume that the value projection matrices
$\mathbf{W}_V^{(h)}\in\mathbb{R}^{d_h\times d}$ for all heads $h\in\{1,\dots,H\}$
are full rank, i.e.,
\begin{equation}
\mathrm{rank}(\mathbf{W}_V^{(h)}) = d_h,
\qquad d_h = d/H.
\end{equation}
This ensures that no information loss is introduced by linear value projections.

\paragraph{Assumption 2 (Injective post-attention mapping).}
Let $\phi(\cdot)$ denote the composition of the output projection, residual connection,
and any subsequent nonlinearity (e.g., feed-forward network).
We assume $\phi$ is injective:
\begin{equation}
\phi(\mathbf{z}_1)=\phi(\mathbf{z}_2)\ \Rightarrow\ \mathbf{z}_1=\mathbf{z}_2.
\end{equation}
This assumption is commonly adopted to isolate the expressive power of the attention
mechanism itself~\cite{yun2019transformers}.

\paragraph{Assumption 3 (Sparse topology and bounded hops).}
The augmented graph $\widetilde{G}$ is sparse, with $|\widetilde{E}|=\mathcal{O}(M)$,
and hop budgets $\{n_h\}_{h=1}^H$ are bounded by a small constant independent of $N$ and $M$.

\paragraph{Assumption 4 (Distinct receptive fields).}
There exist at least two heads $h\neq h'$ such that $n_h\neq n_{h'}$.

\subsection{Sufficiency of Structural Information Injection via Masked Attention}

We first show that attention masks derived from $\widetilde{G}$ alone are sufficient to
inject graph structure into the Transformer layer.

\begin{theorem}[Topology-Constrained Information Flow]
\label{thm:topology_flow}
For an attention head $h$ with hop budget $n_h$, the output representation of any token
$i$ after a single attention layer depends exclusively on tokens within the $n_h$-hop
neighborhood of $i$ in the augmented incidence graph $\widetilde{G}$.
\end{theorem}
\noindent (\textit{Proof}: Appendix~\ref{pro:topology_flow})

\begin{corollary}[Sufficiency of Mask-Based Structure Injection]
\label{cor:sufficiency}
Even without positional or structural encodings, the Transformer layer is explicitly
graph-aware, as all information propagation is constrained by the topology of
$\widetilde{G}$ encoded in the attention masks.
\end{corollary}

\subsection{Expressiveness of Multi-$n$-Hop Heads}

We now establish that multiple heads with distinct hop budgets increase the expressive
power of a single Transformer layer.

\begin{theorem}[Strict Expressiveness Gain of Multi-$n$-Hop Heads]
\label{thm:multihead_expressiveness}
Under Assumptions~1--4, a Transformer layer with multiple attention heads having distinct
hop budgets $\{n_h\}_{h=1}^H$ is strictly more expressive than any layer in which all heads
share a common hop budget.
\end{theorem}
\noindent (\textit{Proof}: Appendix~\ref{pro:multihead_expressiveness})

\begin{corollary}[Multi-Scale Representation in a Single Layer]
\label{cor:multiscalerep}
A single Transformer layer with heterogeneous hop budgets can simultaneously model
local and longer-range dependencies, enabling multi-scale graph representations
without increasing depth.
\end{corollary}

Together, \autoref{thm:topology_flow} and its corollary demonstrate that head-specific attention masks alone are sufficient to inject graph structure into a Transformer layer. Furthermore, \autoref{thm:multihead_expressiveness} and its corollary establish that employing multiple heads with distinct hop budgets strictly enhances expressiveness by enabling multi-scale neighborhood aggregation within a single layer. Collectively, these results provide a principled justification for HopFormer as a sparse, topology-aware Transformer that attains increased expressive power without auxiliary encodings or architectural modifications beyond masked self-attention.

\section{Experiments}

\begin{table*}[t]
\centering
\caption{
Test accuracy ($\uparrow$) on node-level benchmark datasets, reported as mean$_{\pm \text{std}}$ over 3 runs.
``$\mathrm{OOM}$'' denotes configurations that resulted in out-of-memory errors for all tested hyperparameter settings.
The \best{best} and \second{second-best} results for each dataset are highlighted.
}
\label{tab:table2}
\resizebox{\textwidth}{!}{
\begin{tabular}{lccccccccc}
\toprule
\textbf{Model} & \textbf{Cora} & \textbf{Citeseer} & \textbf{Pubmed} & \textbf{Chameleon} & \textbf{Squirrel} & \textbf{Actor} & \textbf{Cornell} & \textbf{Texas} & \textbf{Wisconsin} \\
\midrule
\multicolumn{10}{l}{\textbf{Message Passing Neural Networks}} \\
GCN   & $0.7180_{\pm0.0298}$ & $0.6173_{\pm0.0177}$ & $0.7150_{\pm0.0297}$ & $0.4415_{\pm0.0169}$ & $0.2917_{\pm0.0094}$ & $0.2713_{\pm0.0089}$ & $0.3243_{\pm0.0441}$ & $0.4505_{\pm0.0709}$ & $0.4379_{\pm0.0515}$ \\
GAT   & $0.7460_{\pm0.0116}$ & $0.6123_{\pm0.0162}$ & $0.7590_{\pm0.0184}$ & $0.4613_{\pm0.0021}$ & $0.3016_{\pm0.0086}$ & $0.2842_{\pm0.0121}$ & $0.3513_{\pm0.0441}$ & $0.3604_{\pm0.1469}$ & $0.3791_{\pm0.0606}$ \\
APPNP & $0.7713_{\pm0.0101}$ & $0.6617_{\pm0.0266}$ & \best{0.7750$_{\pm0.0046}$} & $0.4956_{\pm0.0100}$ & $0.3333_{\pm0.0088}$ & $0.2798_{\pm0.0017}$ & $0.4414_{\pm0.0127}$ & $0.5405_{\pm0.0000}$ & $0.4641_{\pm0.0092}$ \\
\midrule
\multicolumn{10}{l}{\textbf{Graph Transformers}} \\
Graphtransformer & \best{0.7913$_{\pm0.0052}$} & \second{0.6717$_{\pm0.0062}$} & $0.7420_{\pm0.0099}$ & $0.4006_{\pm0.0055}$ & $0.2767_{\pm0.0077}$ & $0.2873_{\pm0.0006}$ & $0.3874_{\pm0.0127}$ & $0.6036_{\pm0.0255}$ & $0.5229_{\pm0.0462}$ \\
Graphormer       & $0.4323_{\pm0.0455}$ & $0.4360_{\pm0.0140}$ & $\mathrm{OOM}$        & $0.5453_{\pm0.0136}$ & \second{0.4310$_{\pm0.0120}$} & $0.3713_{\pm0.0057}$ & $0.7117_{\pm0.0127}$ & $0.7748_{\pm0.0127}$ & $0.7582_{\pm0.0244}$ \\
SAN              & $0.5403_{\pm0.0540}$ & $\mathrm{OOM}$        & $\mathrm{OOM}$        & $0.4189_{\pm0.0267}$ & $\mathrm{OOM}$        & $\mathrm{OOM}$        & $0.6667_{\pm0.0337}$ & $0.6667_{\pm0.0337}$ & $0.7843_{\pm0.0160}$ \\
GPS              & $0.7150_{\pm0.0051}$ & $0.6227_{\pm0.0246}$ & $0.7327_{\pm0.0184}$ & $0.4905_{\pm0.0092}$ & $0.3535_{\pm0.0062}$ & $0.3721_{\pm0.0117}$ & $0.7297_{\pm0.0221}$ & $0.7478_{\pm0.0127}$ & $0.7974_{\pm0.0092}$ \\
DIFFormer        & $0.7650_{\pm0.0099}$ & $0.6447_{\pm0.0217}$ & $0.7587_{\pm0.0066}$ & $0.4854_{\pm0.0058}$ & $0.3602_{\pm0.0070}$ & $0.3682_{\pm0.0110}$ & $0.6486_{\pm0.0441}$ & $0.7207_{\pm0.0337}$ & $0.7451_{\pm0.0320}$ \\
SpecFormer       & $0.4883_{\pm0.0084}$ & $0.4540_{\pm0.0100}$ & $0.6973_{\pm0.0079}$ & \second{0.5623$_{\pm0.0089}$} & \best{0.4454$_{\pm0.0141}$} & $0.3647_{\pm0.0087}$ & \second{0.7478$_{\pm0.0255}$} & \second{0.7838$_{\pm0.0221}$} & $0.7778_{\pm0.0092}$ \\
Exphormer        & $0.6120_{\pm0.0213}$ & $0.5117_{\pm0.0070}$ & $0.6987_{\pm0.0170}$ & $0.5161_{\pm0.0068}$ & $0.3490_{\pm0.0071}$ & $0.3577_{\pm0.0045}$ & $0.5405_{\pm0.0441}$ & $0.3784_{\pm0.0221}$ & $0.6994_{\pm0.0333}$ \\
GRIT             & $\mathrm{OOM}$        & $\mathrm{OOM}$        & $\mathrm{OOM}$        & $0.4722_{\pm0.0260}$ & $\mathrm{OOM}$        & $\mathrm{OOM}$        & $0.6486_{\pm0.0584}$ & $0.7027_{\pm0.0382}$ & \second{0.8039$_{\pm0.0160}$} \\
NodeFormer       & $0.7103_{\pm0.0175}$ & $0.5903_{\pm0.0204}$ & $0.6953_{\pm0.0144}$ & $0.4788_{\pm0.0058}$ & $0.3442_{\pm0.0063}$ & $0.3520_{\pm0.0047}$ & $0.6577_{\pm0.0459}$ & $0.6937_{\pm0.0255}$ & $0.7386_{\pm0.0403}$ \\
CoBFormer        & $0.7243_{\pm0.0118}$ & $0.6300_{\pm0.0127}$ & $0.7203_{\pm0.0133}$ & $0.5124_{\pm0.0057}$ & $0.3737_{\pm0.0049}$ & $0.3719_{\pm0.0088}$ & $0.7117_{\pm0.0337}$ & $0.7387_{\pm0.0637}$ & $0.7647_{\pm0.0320}$ \\
SGFormer         & $0.7703_{\pm0.0021}$ & $0.6580_{\pm0.0008}$ & $0.7347_{\pm0.0034}$ & $0.5080_{\pm0.0010}$ & $0.3401_{\pm0.0068}$ & \best{0.3746$_{\pm0.0084}$} & $0.7207_{\pm0.0127}$ & $0.7568_{\pm0.0221}$ & $0.7843_{\pm0.0000}$ \\
DeGTA            & $0.7533_{\pm0.0056}$ & $0.6190_{\pm0.0232}$ & $\mathrm{OOM}$        & $0.4920_{\pm0.0045}$ & $0.3164_{\pm0.0079}$ & $0.3507_{\pm0.0067}$ & $0.5315_{\pm0.0459}$ & $0.6126_{\pm0.0127}$ & $0.5229_{\pm0.0092}$ \\
\midrule
\rowcolor{blue!8}
\textbf{Ours (HopFormer)} & \second{0.7850$_{\pm0.0125}$} & \best{0.6850$_{\pm0.0346}$} & \second{0.7630$_{\pm0.0122}$} & \best{0.5738$_{\pm0.0278}$} & $0.3748_{\pm0.0033}$ & \second{0.3736$_{\pm0.0016}$} & \best{0.7657$_{\pm0.0412}$} & \best{0.7937$_{\pm0.0270}$} & \best{0.8169$_{\pm0.0113}$} \\
\bottomrule
\end{tabular}
}
\end{table*}

\begin{table}[t]
\centering
\caption{
Performance on graph-level benchmark datasets, reported as mean$_{\pm \text{std}}$ over 3 runs.
``$\mathrm{OOM}$'' denotes configurations that resulted in out-of-memory errors for all tested hyperparameter settings.
The \best{best} and \second{second-best} results for each dataset are highlighted.
}
\label{tab:table3}
\resizebox{\linewidth}{!}{
\begin{tabular}{lccccc}
\toprule
\textbf{Model} & \textbf{OGBG-MolHIV} & \textbf{OGBG-MolPCBA} & \textbf{Peptides-Func} & \textbf{Peptides-Struct} & \textbf{ZINC} \\
 & AUC$\uparrow$ & AP$\uparrow$ & AP$\uparrow$ & MAE$\downarrow$ & MAE$\downarrow$ \\
\midrule
\multicolumn{6}{l}{\textbf{Message Passing Neural Networks}} \\
GCN              & $0.6888_{\pm0.0127}$ & $0.0953_{\pm0.0012}$ & $0.3614_{\pm0.0036}$ & $0.4347_{\pm0.0033}$ & $0.6090_{\pm0.0155}$ \\
GAT              & $0.7385_{\pm0.0300}$ & $0.1012_{\pm0.0015}$ & $0.3719_{\pm0.0041}$ & $0.4231_{\pm0.0049}$ & $0.6225_{\pm0.0096}$ \\
APPNP            & $0.7550_{\pm0.0000}$ & $0.0973_{\pm0.0011}$ & $0.3964_{\pm0.0022}$ & $0.4374_{\pm0.0050}$ & $0.6619_{\pm0.0067}$ \\
\midrule
\multicolumn{6}{l}{\textbf{Graph Transformers}} \\
Graphtransformer & $0.6350_{\pm0.0112}$ & $0.0968_{\pm0.0001}$ & $0.3622_{\pm0.0036}$ & $0.4345_{\pm0.0026}$ & $0.6038_{\pm0.0151}$ \\
Graphormer       & $\mathrm{OOM}$        & $\mathrm{OOM}$        & $\mathrm{OOM}$        & $\mathrm{OOM}$        & $0.1305_{\pm0.0072}$ \\
SAN              & $\mathrm{OOM}$        & $\mathrm{OOM}$        & $\mathrm{OOM}$        & $\mathrm{OOM}$        & $0.1341_{\pm0.0063}$ \\
GPS              & \second{$0.7715_{\pm0.0135}$} & $\mathrm{OOM}$        & \best{0.6565$_{\pm0.0038}$} & \second{0.2512$_{\pm0.0008}$} & $0.1968_{\pm0.0120}$ \\
GPS+RWSE         & $0.7713_{\pm0.0105}$ & \best{0.2888$_{\pm0.0019}$} & $0.6427_{\pm0.0035}$ & $0.2851_{\pm0.0012}$ & \second{0.0681$_{\pm0.0004}$} \\
GPS+GE           & $\mathrm{OOM}$        & $\mathrm{OOM}$        & $\mathrm{OOM}$        & $\mathrm{OOM}$        & $0.1789_{\pm0.0071}$ \\
DIFFormer        & $0.7169_{\pm0.0103}$ & $0.0785_{\pm0.0021}$ & $0.3810_{\pm0.0013}$ & $0.4541_{\pm0.0005}$ & $0.6602_{\pm0.0294}$ \\
Exphormer        & $\mathrm{OOM}$        & $\mathrm{OOM}$        & $0.4003_{\pm0.0058}$ & $0.3297_{\pm0.0085}$ & $0.1905_{\pm0.0276}$ \\
GRIT             & $0.7608_{\pm0.0165}$ & \second{0.1931$_{\pm0.0018}$} & $\mathrm{OOM}$        & $\mathrm{OOM}$        & \best{0.0607$_{\pm0.0000}$} \\
GraphMLPMixer    & $0.6889_{\pm0.0146}$ & $0.1386_{\pm0.0018}$ & $0.4225_{\pm0.0095}$ & $0.3065_{\pm0.0019}$ & $0.3617_{\pm0.0222}$ \\
SGFormer         & $0.6794_{\pm0.0076}$ & $0.0933_{\pm0.0040}$ & $0.3726_{\pm0.0018}$ & $0.4492_{\pm0.0002}$ & $0.6703_{\pm0.0033}$ \\
\midrule
\rowcolor{blue!8}
\textbf{Ours (HopFormer)} & \best{0.7997$_{\pm0.0078}$} & 0.1849$_{\pm0.0116}$ & \second{0.6467$_{\pm0.0124}$} & \best{0.2503$_{\pm0.0016}$} & 0.1306$_{\pm0.0194}$ \\
\bottomrule
\end{tabular}
}
\end{table}

\begin{figure}[t]
    \centering
    \includegraphics[width=1.0\linewidth]{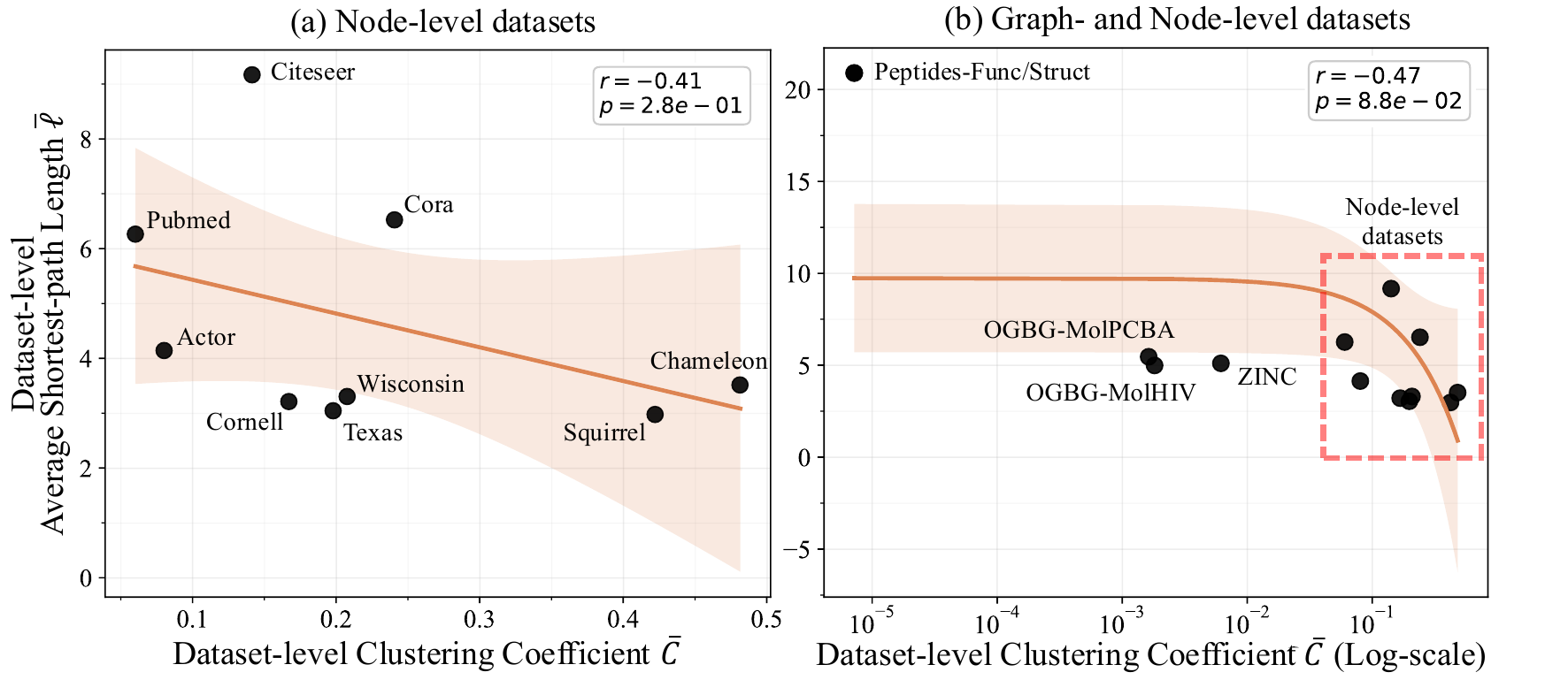}
    \caption{
        Small-world measures of (a) node-level and (b) graph-level benchmark datasets.
        Graph-level datasets generally exhibit weaker small-world characteristics than node-level datasets.
        The shaded region denotes the standard deviation of the linear fit.
    }
    \label{fig:Cvsl}
\end{figure}

\begin{figure}[t]
    \centering
    \includegraphics[width=1.0\linewidth]{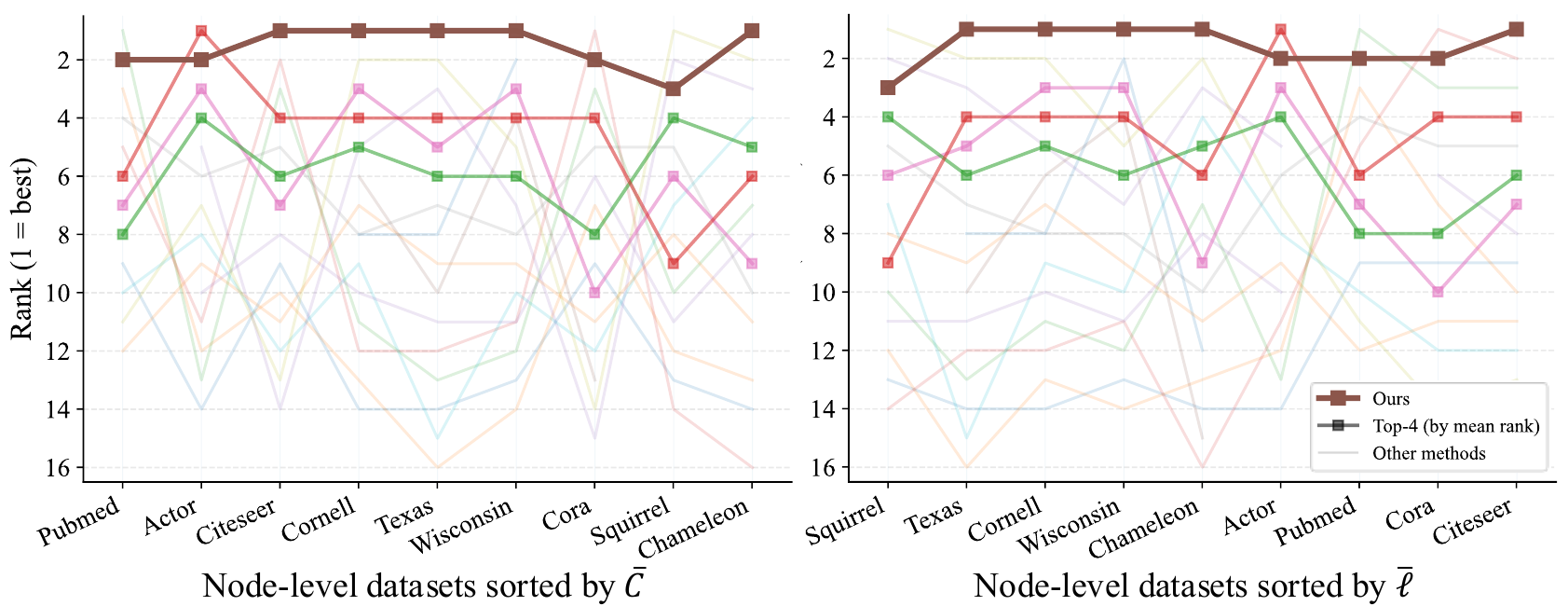}
    \caption{
        Rank trajectories of different methods across node-level datasets.
        Our method maintains consistently strong ranks, while other approaches exhibit larger variability as datasets' small-world characteristics change.
    }
    \label{fig:rank_trajectories}
\end{figure}

\begin{figure*}[t]
    \centering
    \includegraphics[width=0.8\linewidth]{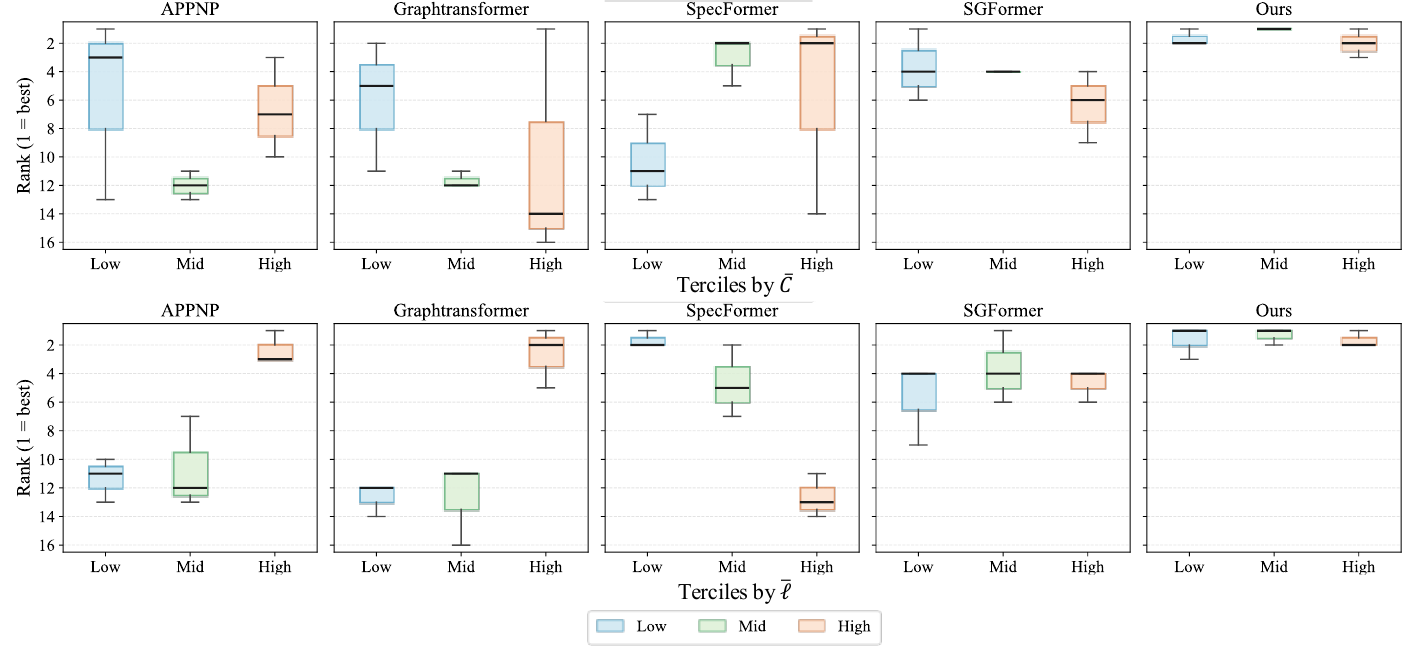}
    \caption{
        Per-method rank distributions across low, mid, and high terciles of small-world measures:
        clustering coefficient $\bar{C}$ (top) and average shortest-path length $\bar{\ell}$ (bottom).
        Lower ranks indicate better performance.
        Boxes indicate the interquartile range, center lines denote the median, and whiskers show the data range (excluding outliers).
    }
    \label{fig:groupedrank}
\end{figure*}

\begin{figure}[t]
    \centering
    \includegraphics[width=1.0\linewidth]{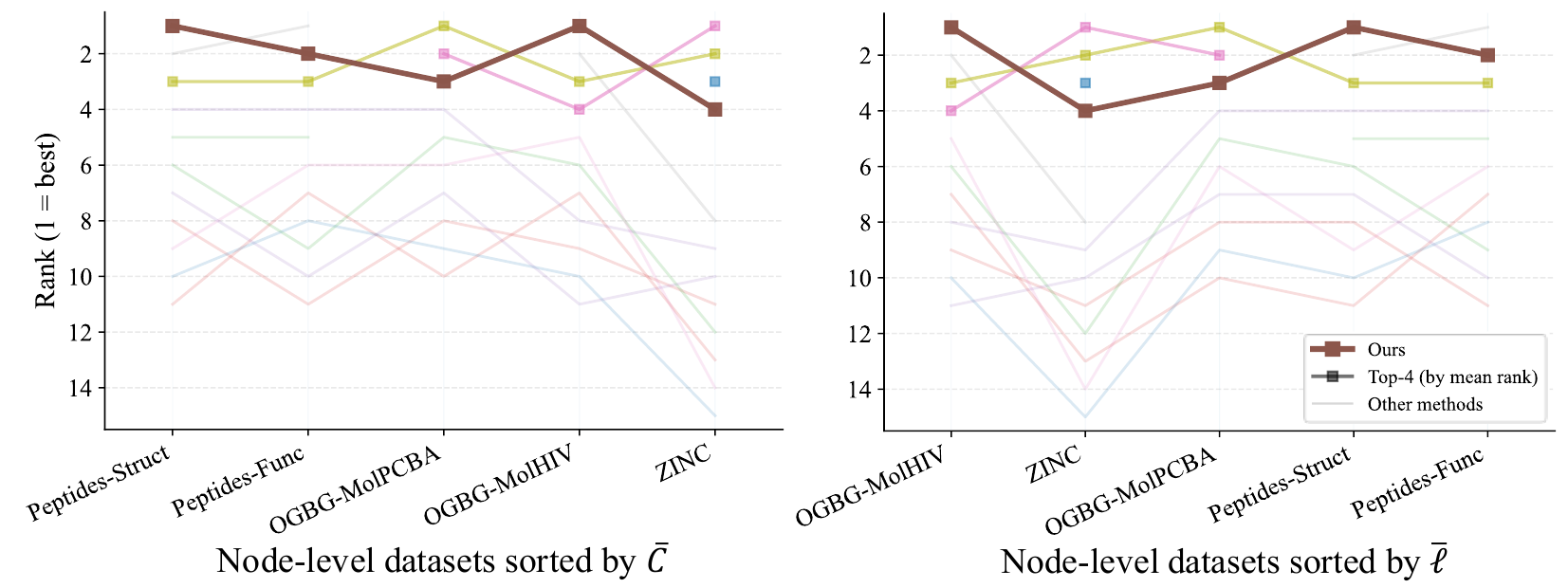}
    \caption{
        Rank trajectories of different methods across graph-level datasets.
        In contrast to node-level benchmarks, the top-performing methods exhibit consistently similar ranks, indicating that when small-world effects are weaker, Transformer-based models tend to achieve comparable performance.
    }
    \label{fig:graph_rank_trajectories}
\end{figure}

\begin{figure}[t]
    \centering
    \includegraphics[width=0.6\linewidth]{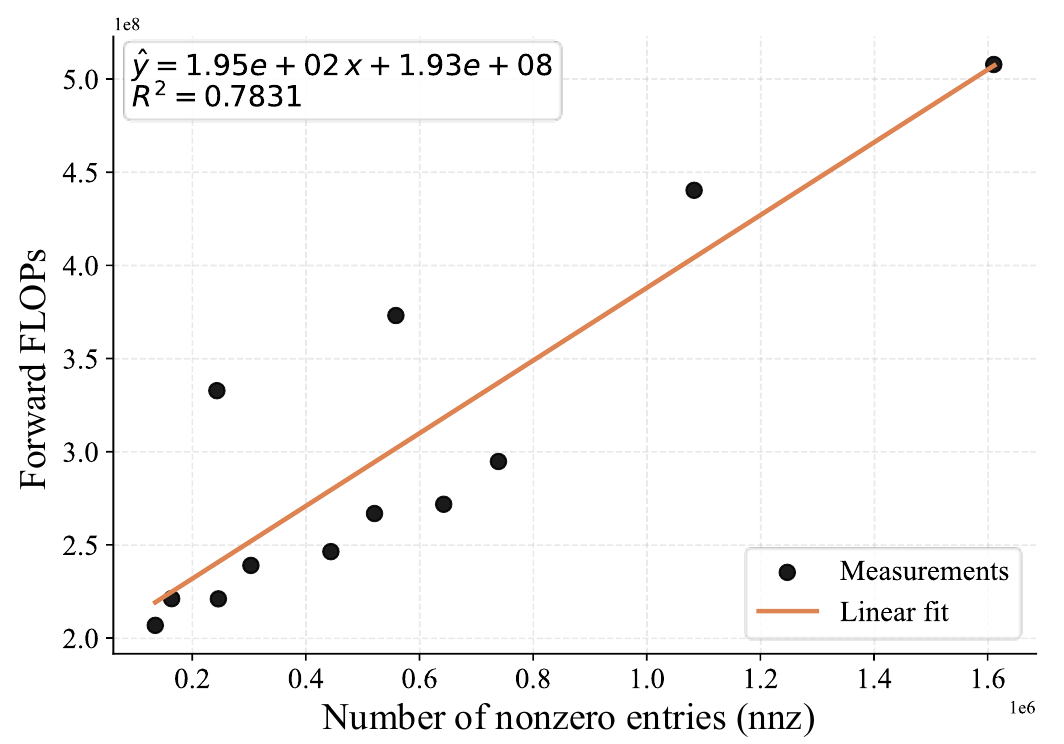}
    \caption{
        Forward-pass FLOPs as a function of the number of nonzero entries in the $n$-hop attention masks.
        Results are obtained on Cornell, Texas, and Wisconsin using single-layer models with different head hop configurations ($\{3,6,12,24\}$, $\{3,3,6,12\}$, $\{3,3,3,6\}$, and $\{3,3,3,3\}$).
        The linear trend indicates that computation scales proportionally with mask sparsity.
    }
    \label{fig:nnz_vs_Flops}
\end{figure}

\subsection{Experimental Setting}

We follow the standardized and fully reproducible experimental protocols established in OpenGT~\cite{tang2025opengt}, ensuring fair and consistent comparisons across models, datasets, and tasks.

\paragraph{Baselines.}
We compare against a comprehensive set of state-of-the-art Graph Transformer architectures, including
DIFFormer~\cite{wu2023difformer}, NodeFormer~\cite{wu2022nodeformer}, GraphGPS~\cite{rampasek2022graphgps}, Graphormer~\cite{ying2021transformers}, GRIT~\cite{ma2023grit}, SGFormer~\cite{wu2023sgformer}, SAN~\cite{kreuzer2021rethinking}, SpecFormer~\cite{bo2023specformer}, Exphormer~\cite{shirzad2023exphormer}, GraphTransformer~\cite{dwivedi2020generalization}, GraphMLPMixer~\cite{he2023generalization}, CoBFormer~\cite{xing2024less}, and DeGTA~\cite{wang2025graph}.
In addition, we include representative message-passing GNNs—GCN~\cite{kipf2016semi}, GAT~\cite{velivckovic2017graph}, and APPNP~\cite{gasteiger2018predict}—to contrast Transformer-based models with local aggregation baselines.

\paragraph{Datasets.}
We evaluate on a diverse suite of node-level and graph-level benchmarks summarized in Appendix~\autoref{tab:opengt_datasets}.
Node-level tasks span citation networks (Cora, Citeseer, Pubmed~\cite{sen2008collective}), web graphs (Chameleon, Squirrel~\cite{rozemberczki2019gemsec}), actor co-occurrence graphs (Actor~\cite{shchur2018pitfalls}), and heterophilous university web networks (Texas, Cornell, Wisconsin~\cite{pei2020geom}).
Graph-level tasks include molecular property prediction (ZINC~\cite{irwin2012zinc}, OGBG-MolHIV, OGBG-MolPCBA~\cite{hu2020open}) and peptide benchmarks from LRGB (Peptides-Func, Peptides-Struct~\cite{dwivedi2022long}).
These datasets exhibit wide variation in size, sparsity, homophily, and small-world characteristics, enabling a thorough evaluation across structural regimes.

\paragraph{Implementation and Hyperparameters.}
All baseline models are implemented using the OpenGT codebase~\cite{tang2025opengt}.
For each baseline, we strictly follow the same hyperparameter search spaces and training protocols as OpenGT, summarized in Appendix~\autoref{tab:opengt_hparams}.
Our method introduces one additional architectural hyperparameter: the per-head $n$-hop receptive field.
Unless otherwise specified, we fix the number of attention heads to four and search over head hop values $n \in \{1, 3, 6, 12, 24, 48\}$.
All experiments are conducted on an NVIDIA GeForce RTX~3090 GPU with 24\,GB memory, and all reported results correspond to the best-performing configuration within this unified search space, ensuring a fair and controlled comparison.

\subsection{Experiment Results}

We empirically evaluate our approach across a wide range of node-level and graph-level benchmarks to answer three core questions:
(i) Is explicit positional or structural encoding necessary for effective graph Transformers?
(ii) Is global dense attention always required?
(iii) Does the computational complexity of our sparse attention formulation follow the theoretical analysis?

\paragraph{Is Explicit Structural Encoding Necessary?}
We first examine whether injecting graph structure via explicit positional encodings or architectural modifications is essential for strong performance.
\autoref{tab:table2} and \autoref{tab:table3} summarize results on node-level and graph-level benchmarks, respectively.
\
Across node-level datasets, our method consistently achieves top or near-top performance, outperforming or matching state-of-the-art graph Transformers that rely on sophisticated structural encodings (e.g., Laplacian eigenvectors, random-walk features, or spectral bases).
Notably, our approach requires none of these components and instead injects graph structure solely through $n$-hop mask-guided sparse attention.
This demonstrates that explicitly controlling the receptive field via sparsity is sufficient to encode meaningful graph structure, without resorting to complex positional encodings or architectural augmentations.
\
A similar trend holds for graph-level benchmarks in \autoref{tab:table3}.
Despite its simplicity, our method achieves the best overall performance on OGBG-MolHIV and Peptides-Struct, while remaining competitive on other datasets.
These results provide strong empirical evidence that explicit positional or structural encodings are not a prerequisite for effective graph Transformer models.

\paragraph{Is Global Dense Attention Always Required?}
We next investigate whether global dense attention is universally beneficial across graphs with different structural properties.
To this end, we first analyze the small-world characteristics of each benchmark.
\
\autoref{fig:Cvsl} visualizes the clustering coefficient $\bar{C}$ and average shortest-path length $\bar{\ell}$ for both node-level and graph-level datasets, where $\bar{C}$ and $\bar{\ell}$ are computed by averaging the per-graph $C$ and $\ell$ values within each dataset.
Node-level benchmarks generally exhibit strong small-world properties, characterized by high clustering and short path lengths, whereas graph-level datasets show substantially weaker small-world effects.
\
We then analyze method behavior across this spectrum.
On node-level datasets, where small-world structure is pronounced, rank trajectories in \autoref{fig:rank_trajectories} reveal that many existing methods exhibit large performance fluctuations as dataset structure changes.
This instability stems from their fixed or implicit receptive fields, which cannot adapt to varying graph connectivity patterns.
In contrast, our method maintains consistently strong ranks across all datasets.
\
This effect is further quantified in \autoref{fig:groupedrank}, which groups datasets into low, mid, and high terciles according to $\bar{C}$ and $\bar{\ell}$.
Among all methods that achieve rank~1 on at least one benchmark, only our approach exhibits both high median performance and narrow variance across all small-world regimes.
This highlights the advantage of explicitly configurable receptive fields.
\
The situation changes for graph-level benchmarks.
As shown in \autoref{fig:graph_rank_trajectories} and \autoref{fig:groupedrank}, when small-world effects are weak, global attention saturates in effectiveness.
The top-performing Transformer-based methods converge to similar ranks, indicating diminishing returns from full global connectivity.
These observations collectively suggest that global dense attention is not universally necessary and that the ability to adaptively control the receptive field is critical for generalization across diverse graph structures.

\paragraph{Computational Efficiency and Scalability}
Finally, we empirically validate our computational complexity analysis.
\autoref{fig:nnz_vs_Flops} plots forward-pass FLOPs against the number of nonzero entries in the $n$-hop attention masks across multiple configurations.
\
The results show a clear linear relationship between FLOPs and mask sparsity, confirming that the computational cost of our method scales proportionally with $\mathrm{nnz}(\mathbf{M})$.
This validates our theoretical analysis and demonstrates that our approach enables flexible trade-offs between accuracy and efficiency.
By adjusting the receptive field according to graph structure, our model can achieve competitive performance with substantially reduced computation when global attention is unnecessary.

\section{Conclusion}
We show that effective graph Transformers do not require explicit positional or structural encodings, nor architectural modifications to the standard Transformer. Injecting topology solely via head-specific $n$-hop masked sparse attention is sufficient to capture graph structure, while preserving architectural simplicity and providing explicit, interpretable control over receptive fields with a clear efficiency–expressivity trade-off. Across diverse node-level and graph-level benchmarks, this minimal design is competitive with or outperforms prior graph Transformers, and exhibits markedly more stable performance on graphs with strong small-world structure. On graph-level benchmarks where small-world effects are weaker, performance saturates and multiple Transformer variants converge to similar ranks, indicating that dense global attention offers diminishing returns.

\paragraph{Limitations and Future Work.}
Our analysis of the link between small-world properties and optimal receptive fields remains primarily empirical. A more thorough characterization---ideally with theory---could explain when and why particular hop budgets are sufficient, and enable automatic selection or learning of $n$-hop configurations. This would reduce reliance on hyperparameter tuning and improve usability and robustness across new domains and scales.


\footnotesize
\section*{Acknowledgements}
This work was supported in part by the DARPA Young Faculty Award, the National Science Foundation (NSF) under Grants \#2127780, \#2319198, \#2321840, \#2312517, \#2431561, and \#2235472, the Semiconductor Research Corporation (SRC), the Office of Naval Research through the Young Investigator Program Award, and Grants \#N00014-21-1-2225 and \#N00014-22-1-2067, Army Research Office Grant \#W911NF2410360. Additionally, support was provided by the Air Force Office of Scientific Research under Award \#FA9550-22-1-0253, along with generous gifts from Xilinx and Cisco.

\normalsize
\bibliography{references}

\clearpage
\setcounter{page}{1}
\section*{Appendix}

\section{Additional Related Work and Background}

\paragraph{Graph Transformers.}
Graph Transformers extend self-attention to graph-structured data by treating nodes—and in some cases edges—as tokens, while injecting structural information to compensate for the lack of canonical ordering. Most existing approaches encode graph structure through explicit positional or structural encodings, including shortest-path-distance and degree biases in Graphormer~\cite{ying2021transformers}, Laplacian eigenvectors or spectral features~\cite{dwivedi2020generalization,kreuzer2021rethinking,zhang2020graph}, or hybrid Transformer--MPNN architectures such as GraphGPS and GRIT that combine global attention with message passing~\cite{rampasek2022graphgps,ma2023grit}. While empirically effective, these methods introduce encoding-dependent inductive biases and typically retain dense global self-attention, resulting in quadratic or higher complexity that is misaligned with graph sparsity. Distance-restricted or sparse attention variants have been explored~\cite{ying2021transformers,kreuzer2021rethinking}, but are usually tied to fixed encodings or applied as secondary constraints, without explicit, per-head control over receptive fields. In contrast, our approach injects graph structure solely through head-specific $n$-hop attention masks, enabling topology-aligned sparse attention with interpretable receptive fields, while preserving the standard Transformer architecture without positional encodings or auxiliary modules.

\section{Proofs}

\subsection{Proof of \autoref{thm:topology_flow}}
\begin{proof}\label{pro:topology_flow}
Let $\mathbf{M}^{(h)}\in\{0,1\}^{(N+M)\times(N+M)}$ denote the head-specific attention mask
defined in \autoref{eq:defmask}.
By construction, $\mathbf{M}^{(h)}_{ij}=1$ if and only if token $j$ is reachable from
token $i$ within at most $n_h$ hops in $\widetilde{G}$.

Under the masked attention formulation in \autoref{eq:dotproduct}, attention scores are
computed only on the support of $\mathbf{M}^{(h)}$. That is, for any $j$ such that
$\mathbf{M}^{(h)}_{ij}=0$, the corresponding query–key interaction is excluded from the
softmax normalization, yielding an effective attention weight
$\alpha^{(h)}_{ij}=0$.

Consequently, the output of head $h$ at token $i$ can be written as
\[
\mathbf{z}^{(h)}_i
=
\sum_{j:\,\mathbf{M}^{(h)}_{ij}=1}
\alpha^{(h)}_{ij}\,\mathbf{W}_V^{(h)}\mathbf{h}^{(0)}_j
=
\sum_{j\in\mathcal{N}_{n_h}(i)}
\alpha^{(h)}_{ij}\,\mathbf{W}_V^{(h)}\mathbf{h}^{(0)}_j,
\]
where $\mathcal{N}_{n_h}(i)$ denotes the $n_h$-hop neighborhood of $i$ in
$\widetilde{G}$. Therefore, the representation $\mathbf{z}^{(h)}_i$ depends only on
tokens within this neighborhood, completing the proof.
\end{proof}

\subsection{Proof of \autoref{thm:multihead_expressiveness}}
\begin{proof}\label{pro:multihead_expressiveness}
Consider two tokens $i$ and $j$ such that
$j\in\mathcal{N}_{n_{h'}}(i)$ but $j\notin\mathcal{N}_{n_h}(i)$
for some $h\neq h'$.
By Theorem~\ref{thm:topology_flow}, head $h'$ can incorporate information from $j$ into
$\mathbf{z}^{(h')}_i$, whereas head $h$ cannot.

Let $\mathbf{z}_i = [\mathbf{z}^{(1)}_i;\dots;\mathbf{z}^{(H)}_i]$ denote the concatenated
multi-head output.
Then $\mathbf{z}_i$ jointly encodes information from multiple, distinct neighborhood
scales.
No single-hop-budget head can replicate this representation:
a smaller hop budget excludes $j$, while a larger hop budget necessarily aggregates
additional nodes beyond $\mathcal{N}_{n_{h'}}(i)$, yielding a different function.

By Assumptions~1 and~2, these differences are preserved through value projections and
post-attention mappings.
Hence, the function class realizable by multi-$n$-hop heads strictly contains that of
uniform-hop heads.
\end{proof}

\section{Additional Tables}

\begin{table}[h]
\centering
\caption{Statistics of the datasets used in the benchmark.}
\label{tab:opengt_datasets}
\resizebox{\linewidth}{!}{
\begin{tabular}{lccccccc}
\toprule
\textbf{Dataset} & \textbf{Level} & \textbf{\#Graphs} & \textbf{Avg. Nodes} & \textbf{Avg. Edges} & \textbf{\#Classes} & \textbf{\#Features} & \textbf{Metric} \\
\midrule
Cora        & Node  & 1 & 2,708  & 5,429   & 7   & 1,433 & Accuracy \\
Citeseer   & Node  & 1 & 3,327  & 4,732   & 6   & 3,703 & Accuracy \\
Pubmed     & Node  & 1 & 19,717 & 44,338  & 3   & 500   & Accuracy \\
Squirrel   & Node  & 1 & 5,201  & 217,073 & 10  & 2,089 & Accuracy \\
Chameleon  & Node  & 1 & 2,277  & 31,421  & 10  & 2,325 & Accuracy \\
Actor      & Node  & 1 & 7,600  & 30,019  & 5   & 931   & Accuracy \\
Texas      & Node  & 1 & 183    & 325     & 5   & 1,703 & Accuracy \\
Cornell    & Node  & 1 & 183    & 298     & 5   & 1,703 & Accuracy \\
Wisconsin  & Node  & 1 & 251    & 515     & 5   & 1,703 & Accuracy \\
\midrule
ZINC              & Graph & 12,000  & 23.2 & 49.8 & --  & 28 & MAE \\
OGBG-MolHIV       & Graph & 41,127  & 25.5 & 27.5 & 2   & 9  & ROC-AUC \\
OGBG-MolPCBA      & Graph & 437,929 & 26.0 & 28.1 & 128 & 9  & AP \\
Peptides-Func     & Graph & 15,535  & 151  & 307  & 10  & 9  & AP \\
Peptides-Struct   & Graph & 15,535  & 151  & 307  & 11  & 9  & MAE \\
\bottomrule
\end{tabular}
}
\end{table}

\begin{table}[h]
\centering
\caption{Hyperparameter search space for all implemented Graph Transformer models.}
\label{tab:opengt_hparams}
\resizebox{\linewidth}{!}{
\begin{tabular}{ll}
\toprule
\textbf{Algorithm} & \textbf{Hyperparameter Search Space} \\
\midrule
\multicolumn{2}{l}{\textbf{General Settings}} \\
Learning rate            & $\{10^{-4}, 3\!\times\!10^{-4}, 10^{-3}, 3\!\times\!10^{-3}, 10^{-2}\}$ \\
Weight decay             & $\{10^{-5}, 3\!\times\!10^{-5}, 10^{-4}, 3\!\times\!10^{-4}, 10^{-3}\}$ \\
Number of layers         & $\{1, 2, 3, 4\}$ \\
Number of attention heads& $\{1, 2, 3, 4\}$ \\
Dropout                  & $\{0, 0.2, 0.5, 0.8\}$ \\
Attention dropout        & $\{0, 0.2, 0.5, 0.8\}$ \\
\midrule
SGFormer~\cite{wu2023sgformer} & Aggregate method: \{add, cat\} \\
                              & Graph weight: \{0.2, 0.5, 0.8\} \\
\midrule
DIFFormer~\cite{wu2023difformer} & Graph weight: \{0.2, 0.5, 0.8\} \\
\midrule
DeGTA~\cite{wang2025graph}       & $K$: \{2, 4, 8\} \\
\midrule
\textbf{Ours} 
& Attention head configuration: fixed to 4 heads; \\
& each head hop $n \in \{1, 3, 6, 12, 24, 48\}$ \\
\bottomrule
\end{tabular}
}
\end{table}

\end{document}